\documentclass[letterpaper, 10pt, conference]{ieeeconf}  

\IEEEoverridecommandlockouts                              

\usepackage{graphics} 
\usepackage{epsfig} 
\usepackage{mathptmx} 
\usepackage{times} 
\usepackage{amsmath} 
\usepackage{amssymb}  

\usepackage[sort]{cite}

\let\labelindent\relax 
\usepackage{enumitem}

\usepackage{amsthm}
\usepackage{color}
\usepackage{booktabs}
\usepackage{mathfont}

\newtheorem{definition}{Definition}

\newtheorem{lemma}{Lemma}

\newtheorem{theorem}{Theorem}

\DeclareMathOperator*{\esssup}{\mathrm{ess\,sup}}

\definecolor{brickred}{rgb}{0.8, 0.25, 0.33}

\title{\LARGE \bf
Lyapunov-Net: A Deep Neural Network Architecture for Lyapunov Function Approximation
}

\author{Nathan Gaby$^{1}$ \quad Fumin Zhang$^{2}$ \quad Xiaojing Ye$^{1}$
\thanks{*The work is supported in part by National Science Foundation under grants DMS-1818886 and DMS-1925263.}
\thanks{$^{1}$Department of Mathematics and Statistics, Georgia State University, Atlanta, GA 30303, USA
        {\tt\small \{ngaby1,xye\}@gsu.edu}}%
\thanks{$^{2}$School of Electrical and Computer Engineering, Georgia Institute of Technology, Atlanta, GA 30332, USA
        {\tt\small fumin@gatech.edu}}%
}

\begin{document}

\maketitle
\thispagestyle{empty}
\pagestyle{empty}

\begin{abstract}
We develop a versatile deep neural network architecture, called Lyapunov-Net, to approximate Lyapunov functions of dynamical systems in high dimensions. Lyapunov-Net guarantees positive definiteness, and thus can be easily trained to satisfy the negative orbital derivative condition, which only renders a single term in the empirical risk function in practice. This significantly simplifies parameter tuning and results in greatly improved convergence during network training and approximation quality. We also provide comprehensive theoretical justifications on the approximation accuracy and certification guarantees of Lyapunov-Nets.
We demonstrate the efficiency of the proposed method on nonlinear dynamical systems in high dimensional state spaces, and show that the proposed approach significantly outperforms the state-of-the-art methods.
\end{abstract}

\section{INTRODUCTION}
\label{sec:intro}

Lyapunov functions play central roles in dynamical systems and control theory. They have been used to characterize asymptotic stability of equilibria of nonlinear systems, estimate regions of attraction, investigate system robustness against perturbations, and more \cite{chang2020neural,grune2020computing,richards2018lyapunov,giesl2015review}.
In addition, Control Lyapunov functions are used to derive stabilizing feedback control laws \cite{anand2021lyapunov}. 
%
%
The properties that Lyapunov functions must satisfy can be viewed as constrained partial differential inequations (PDIs). 
%
%
However, like partial differential equations (PDEs), for problems with high dimensional state spaces, it remains very challenging to compute approximated Lyapunov functions using classical numerical methods, such as finite-difference method (FDM) and finite element method (FEM). 
The main issue is the sizes of variables to solve grows exponentially fast in terms of the state space dimension of the problem using these classical methods.
Such issue is known as the curse of dimensionality \cite{bellman1957dynamic}. 

Recent years have witnessed a tremendous success of deep learning (DL) methods in solving high-dimensional PDEs \cite{e2018deep,han2017overcoming,raissi2017physics,zang2020weak}. Backed by the provable approximation power \cite{hornik1991approximation,hornik1989multilayer,liang2017deep,petersen2018optimal,yarotsky2017error}, a (deep) neural network can approximate a general set of functions at any prescribed accuracy level if the network size (e.g., width and depth) is sufficiently large. Successful training of the parameters (e.g., weights and biases) of the network approximating the solution of a PDE also require a properly designed loss function and an efficient (nonconvex) optimization algorithm. However, unlike traditional DL-based methods in many classification and regression applications, the methods developed in \cite{e2018deep,han2017overcoming,raissi2017physics,zang2020weak} parameterize the solutions of specific PDEs as deep networks, and the training of the these networks only require function and partial derivative (up to certain order depending on the PDE's order) evaluations at randomly sampled collocation points.

Recently, deep neural networks (DNNs) also emerged to approximate Lyapunov functions of nonlinear and non-polynomial dynamical systems in high-dimensional spaces 
\cite{chang2020neural,grune2020computing,richards2018lyapunov,wei2021barrierNet,dawson2021SafeNC}. For control Lyapunov functions, DNNs can also be used to approximate the control laws, eliminating the restrictions on control laws to specific function type (e.g., affine functions) in classical control methods such as linear-quadratic regulator (LQR).


In this work, we propose a general framework to approximate Lyapunov functions and control laws using deep neural nets. Our main contributions lie in the following aspects:
\begin{itemize}[leftmargin=*]
    \item We propose a highly versatile network architecture, called \emph{Lyapunov-Net}, to approximate Lyapunov functions. Specifically, this network architecture guarantees the desired positive definiteness property. This leads to simple parameter tuning, significantly accelerated convergence during network training, and greatly improved solution quality. This will allow for fast adoption of neural networks in (control) Lyapunov function approximation in a broad range of applications.
    \item We show that the proposed Lyapunov-Nets are dense in a general class of Lyapunov functions. More importantly, we prove that the Lyapunov-Nets trained by minimizing empirical risk functions using finitely many collocation points are Lyapunov functions, which provides a theoretical certification guarantee of neural network based Lyapunov function approximation. 
    \item We test the proposed Lyapunov-Net to solve several benchmark problems numerically. We show that our method can effectively approximate Lyapunov functions in these problems with very high state dimensionality. Moreover, we demonstrate that our method can be used to find control laws in the control Laypunov problem setting.
\end{itemize}
With the accuracy and efficiency of Lyapunov-Nets demonstrated in this work, we anticipate that our method can be applied to a much broader range of dynamical system and control problems in high dimensional spaces.



The remainder of the paper is organized as follows. In Section \ref{sec:related}, we review the recent developments in approximating Lyapunov functions and control laws using deep learning methods as well as briefly discuss some results in neural network approximation power. In Section \ref{sec:proposed}, we present our proposed network architecture and training strategies as well as critical theory providing certification guarantees for Lyapunov-Net in approximating Lyapunov functions. In Section \ref{sec:results}, we conduct numerical experiments using the proposed method on several nonlinear dynamical systems, demonstrating promising results on these tests. Section \ref{sec:conclusion} concludes this paper.

\section{RELATED WORK}
\label{sec:related}

In this section, we first  discuss the recent progresses in approximating Lyapunov functions and control laws using DNNs. Then we provide an overview of the universal approximation theory of (deep) neural networks which informs our theory to follow.

\paragraph{Approximating Lyapunov function using neural networks}
Approximating Lyapunov functions using neural networks can be dated back to \cite{long1993feedback,sontag1991feedback}. 
In \cite{long1993feedback} the authors attempted the idea assuming that a shallow neural network can exactly represent the target Lyapunov function. In \cite{sontag1991feedback}, stabilization problems in control using neural networks with one or two hidden layers are considered. 
In \cite{petridis2006construction}, the authors propose a special shallow network to approximate Lyapunov functions, however, the determination of the network parameters involves a series of constraints and Hessian computations.
In \cite{khansari2014learning}, the control Lyapunov functions (CLF) using quadratic Lyapunov function candidate is considered. A DNN approach to CLF is considered in \cite{abate2021neural,chang2020neural} which are similar.
Approximating stabilizing controllers using neural networks is considered in  \cite{lewis2020neural}.
Time discretized dynamics and successive parameter updates are considered in  \cite{richards2018lyapunov,serpen2005empirical}. 
Specifically, \cite{richards2018lyapunov} considers jointly learning the Lyapunov function and its decreasing region which is expected to match the Region of Attraction (ROA).   
%
%

Special network architectures to approximate Lyapunov functions are also considered in \cite{richards2018lyapunov}. In \cite{richards2018lyapunov}, a Lyapunov neural network of form $\|\phi_{\theta}(\cdot)\|^2$ is proposed to ensure positive semi-definiteness, where $\phi_{\theta}(\cdot)$ is a DNN. To ensure positive definiteness, the authors restrict $\phi_{\theta}$ to be a feed-forward neural network, all weight matrices to have full column rank, all biases to be zero, and all activation functions to have trivial null space (e.g., tanh or leaky ReLU but not sigmoid, swish, softmax, ReLU, or RePU). Compared to the architecture in \cite{richards2018lyapunov}, the network architecture in the present work does not have any of these restrictions. Moreover, the positive definite weight matrix constructions in \cite{richards2018lyapunov} can make $\|\phi_{\theta}(x)\|^2$ grow excessively fast as $\|x\|$ increases, whereas ours does not have this issue.

In \cite{grune2020computing},  the author considers dynamical systems with small-gain property, which yields a compositional Lyapunov function structure that have decoupled components. In this case, it is shown that the size of the DNN used to approximate such Lyapunov functions can be dependent exponentially on the maximal dimension of these components rather than the original state space dimension. 

Still other authors propose neural networks for other control certificates such as barrier functions \cite{robey2020control, wei2021barrierNet,dawson2021SafeNC} and contraction metrics \cite{tsukamoto2021contraction, tsukamoto2021stability}. While these are a different direction compared to our approach, they present other promising angles for neural networks in control.

\paragraph{Universal approximation theory of neural networks} 
The approximation powers of neural networks have been studied since the early 1990s. The asymptotic analysis justifies that shallow networks with one hidden layer and any non-polynomial differentiable activation function is dense in $C^1(\Omega)$ for any compact set $\Omega \subset \mathbb{R}^{d}$ \cite{hornik1991approximation}. However, the number of neurons needed to obtain a prescribed approximation accuracy of $\epsilon> 0$ may grow exponentially fast in $d$. Specifically, for shallow networks with sigmoid activation function, it is shown that a universal approximator of $C^1([-1,1]^{d})$ needs $O(\epsilon^{-d})$ computing neurons \cite{hornik1991approximation,hornik1989multilayer}. The results are extended to (deep) ReLU neural networks in \cite{liang2017deep,yarotsky2017error}, and RePU networks in \cite{li2020repu}. 

%
There are also works developed to investigate the power of DNNs in approximating other classes of functions. In the class of analytic functions, authors in \cite{e2018exponential, opschoor2021nnholomorphic} showed the exponential convergence speed based on the size of the approximating network. 
The study of approximation power of DNN remains an active research field \cite{beknazaryan2021regDNN,blanchard2021breaking, montanelli2019relu}, as the theory continues to lag behind the success of neural networks in practice \cite{guhring2020approximation}.
%

\section{PROPOSED METHOD}
\label{sec:proposed}
In this section, we propose the Lyapunov-Net architecture to approximate Lyapunov functions, and discuss its key properties and associated training strategies. We first recall the definition of Lyapunov function for a given general dynamical system $x'=f(x)$ with Lipschitz continuous $f$ on the problem domain $\Omega$. We assume that $x=0$ is the unique equilibrium of the system dynamics that lies within $\Omega$.
If $x=0$ is asymptotically stable, then by the converse Lyapunov theory, we can find a Lyapunov function $V$ defined as below. 
\begin{definition}
[Lyapunov function]
\label{def:LF}
Let $\Omega \subset \mathbb{R}^{d}$ be a bounded open set and $0 \in \Omega$, and $f:\Omega \to \mathbb{R}^d$ a Lipschitz function. Then $V: \Omega \to \mathbb{R}$ is called a \emph{Lyapunov function} if (i) $V$ is positive definite, i.e., $V(x) \ge 0$ for all $x \in \Omega$ and $V(x)=0$ if and only if $x=0$; and (ii) $V$ has negative orbital-derivative, i.e., $DV(x) \cdot f(x) < 0$ for all $x \ne 0$. 
\end{definition}

Denote $B(x;\delta):=\{y \in \mathbb{R}^d: \|y - x\| < \delta\}$ as the open ball of radius $\delta>0$ centered at $x$, and $\Omega_{\delta}:=\Omega \setminus B(0;\delta)$ the problem domain with $B(0;\delta)$ excluded.
%
%
Our approach is to approximate a Lyapunov function using a specially designed deep neural network. Due to limited network size and finite collocation points for training in practice, we only guarantee that our approximation function $V$ is positive definite in $\Omega$ and satisfies a slightly relaxed condition of Definition \ref{def:LF} (ii) as follows:
\begin{equation}
    D V (x) \cdot f(x) < 0,\ \mbox{ for all } x \in \Omega_{\delta}.
\end{equation}
where $\delta>0$ is arbitrary and prescribed by user.
We term such a function $V$ as a \emph{$\delta$-accurate Lyapunov function}.

A $\delta$-accurate function can be used as a Lyapunov function for $x'=f(x)$ (or control-Lyapunov function for $f(x,u(x))$ where the control $u$ is to be found jointly with the Lyapunov function, see Section \ref{subsec:application}) to prove that the solution $x(t)$ will be \emph{ultimately bounded} within a small compact set (i.e., $\overline{B(0;\delta)}$) \cite{Khalil01}. 
As $\delta\to 0$, the size of this compact set will converge to $0$, hence asymptotic stability is established.
In practice, the smaller the value of $\delta$, the larger network size and number of collocation points are needed to train such a $\delta$-accurate Lyapunov function.

%
%
Our method will build a versatile deep network architecture that is particularly suitable for finding $V_\theta$ for dynamics evolving in high state dimension $d$. 
%
%
We also demonstrate that the training of Lyapunov-Net renders a minimization problem of a simple risk function, and thus requires much less manual hyper-parameter tuning and achieves high optimization efficiency during practical computations. 
We will also provide theoretical justifications that such networks can approximate a large class of Lyapunov functions while maintaining bounded Lipschitz constants during training, ensuring robust certification guarantees.

\subsection{Lyapunov-Net and its properties}
%

We first construct an arbitrary network $\phi_{\theta}(\cdot): \mathbb{R}^d \to \mathbb{R}$ with the set of all its $m$ trainable parameters denoted by $\theta \in \mathbb{R}^{m}$. This network has input dimension $d$ and output dimension 1. Then we build a scalar-valued network $V_{\theta}:\mathbb{R}^{d} \to \mathbb{R}$ from $\phi_{\theta}$ as follows:
\begin{equation}
\label{eq:V}
    V_{\theta}(x):=|\phi_{\theta}(x)-\phi_{\theta}(0)| + \bar{\alpha}\|x\|,
\end{equation}
where $\bar{\alpha} > 0$ is a small user-chosen parameter and $\|\cdot\|$ is the standard 2-norm in Euclidean space. Then it is easy to verify that $V_{\theta}(0) = 0$ and 
\begin{equation*}
    V_{\theta}(x) \ge \bar{\alpha}\|x\| > 0, \quad \forall\, x \ne 0.
\end{equation*}
In other words, $V_{\theta}$ is a \emph{candidate} Lyapunov function that already satisfies condition (i) in Definition \ref{def:LF}: \emph{for any network structure $\phi_{\theta}$ with any parameter $\theta$}, $V_{\theta}$ is positive definite and only vanishes at the equilibrium $0$. We call the neural network $V_{\theta}$ with architecture specified in \eqref{eq:V} a \emph{Lyapunov-Net}. 

We make several remarks regarding the Lyapunov-Net architecture \eqref{eq:V} below.

First, we use the augment term $\bar{\alpha}\|x\|$ to lower bound the function $V_{\theta}(x)$ in order to ensure positive definiteness. Other positive definite function $r: \mathbb{R}^d \to \mathbb{R}_+$ such that $r(x)=0$ if and only if $x =0$ can be chosen as such lower bound, such as $\bar{\alpha}\|x\|^2$ or $\bar{\alpha}\log(1+\|x\|^2)$, etc.
%

Second, the term $|\phi_{\theta}(x)-\phi_{\theta}(0)|$ in \eqref{eq:V} can be replaced with $\psi(\phi_{\theta}(x)-\phi_{\theta}(0))$ for any non-negative function $\psi: \mathbb{R}^{m} \to \mathbb{R}_+$ with $\psi(0)=0$. We chose $\psi(\cdot) = |\cdot|$ for its application to our theory and simplicity. Note that we can also use vector-valued DNN $\phi_{\theta} : \Rbb^d \to \Rbb^{d'}$ where $d'$ is arbitrary which can further improve network capacity. So long as $\phi_{\theta}$ is Lipschitz continuous then $V_{\theta}$ is Lipschitz continuous and hence weakly differentiable. In practice, we can also use $\|\cdot\|^2$ or Huber norm to smooth out $V_{\theta}$.

Third, we will provide a detail characterization of the $\bar{\alpha}>0$ in Section \ref{subsec:certification}. However, in practice, $\bar{\alpha}$ is free as Lyapunov functions can be arbitrarily scaled by a positive constant. 

Fourth, if the equilibrium is at $x^*$ instead of $0$, then we can simply replace $\phi_{\theta}(0)$ and $\|x\|$ in \eqref{eq:V} with $\phi_{\theta}(x^*)$ and $\|x - x^*\|$, respectively. Without loss of generality, we assume the equilibrium is at $0$ hereafter in this paper. 

The properties remarked above show that $V_{\theta}$ defined in \eqref{eq:V} serves as a versatile network architecture for approximating Lyapunov functions. As we will show soon, this architecture significantly eases network training and yields accurate approximation of Lyapunov functions in practice.

\subsection{Training of Lyapunov-Net}
The training of Lyapunov-Net $V_{\theta}$ in \eqref{eq:V} refers to finding a specific network parameter $\theta$ such that the negative-orbital-derivative condition $D V_{\theta}(x) \cdot f(x) < 0 $ is satisfied at every $x \in \Omega \setminus \{0\}$.
%
This can be achieve by minimizing a risk function that penalize $V_{\theta}$ if the negative-orbital-derivative condition fails to hold at some $x$.
We can choose the following as such risk function:
\begin{equation}\label{eq:Lyapunov-risk}
    \ell(\theta):= \frac{1}{|\Omega|}\int_{\Omega} ( DV_{\theta}(x)\cdot f(x) +\gamma\|x\|)_+^2 \, dx,
\end{equation}
where $(z)_+:= \max(z,0)$ for any $z \in \mathbb{R}$.
%
%
Here $\gamma$ is a user chosen parameter. It is clear that the risk function $\ell(\theta)$ reaches the minimal function value $0$ if and only if $DV_{\theta}(x) \cdot f(x) \le -\gamma \|x\|$ for all $x \in \Omega$, which, in conjunction with $V_{\theta}$ already being positive definite, ensures that $V_{\theta}$ is a Lyapunov function.

In practice, the integral in \eqref{eq:Lyapunov-risk} does not have analytic form, and we have to resort to 
%
%
Monte-Carlo integration which is suitable for high-dimensional problems. To this end, we notice that $\ell(\theta)= \mathbb{E}_{X \sim U(\Omega)}[(DV_{\theta}(X)\cdot f(X)+\gamma\|x\|)_+^2]$ where $U(\Omega)$ stands for the uniform distribution on $\Omega$ and hence has distribution $1/|\Omega|$. Therefore, we approximate $\ell(\theta)$ in \eqref{eq:Lyapunov-risk} using the empirical expectation
\begin{equation}\label{eq:empirical-Lyapunov-risk}
    \hat{\ell}(\theta):= \frac{1}{N} \sum_{i=1}^{N} ( DV_{\theta}(x_i)\cdot f(x_i) +\gamma\|x_i\|)_+^2,
\end{equation}
where $\{x_i: i\in[N]\}$ are independent and identically distributed (i.i.d.) samples from $U(\Omega)$. Then we train the Lyapunov-Net $V_{\theta}$ by minimizing $\hat{\ell}(\theta)$ in \eqref{eq:empirical-Lyapunov-risk} with respect to $\theta$. In this case, standard network training algorithms, such as ADAM \cite{kingma2014adam:}, can be employed in conjunction with automatic differentiation to calculate gradients. 
Note that techniques to improve the efficiency of Monte-Carlo integration, such as importance sampling, can be applied.
For simplicity, we use uniform sampling in the experiments, and leave improved sampling strategies for future investigation.

The empirical risk function defined using finitely many sampling points introduces inaccuracies near $0$, which is common in the literature.
Several existing works \cite{chang2020neural,grune2020computing} observed that the deep-learning-based Lyapunov function approximation may violate the condition $DV_{\theta}(x) \cdot f(x) < 0$ within a small neighborhood of the equilibrium. As such, in this work we will concern ourselves with finding a network which satisfies the Lyapunov conditions everywhere except possible in some small neighborhood around the equilibrium. Hence our Lyapunov function will establish ultimate boundedness of the solution $x(t)$, which means that $x(t)$
converges to a small neighborhood $\overline{B(0;\delta)}$ of $0$. We note that $\delta>0$ can be set arbitrarily small at the expense of larger network size of $V_{\theta}$ and greater amount of sampling points. In this case, we can replace $\Omega$ with $\Omega_{\delta}$ in \eqref{eq:Lyapunov-risk} and exclude the points in $B(0;\delta)$ in \eqref{eq:empirical-Lyapunov-risk}. In practical implementation, we just apply standard uniform sampling without such exclusion for simplicity.

It is worth stressing that the main advantage of the Lyapunov-Net architecture \eqref{eq:V} is that the risk function \eqref{eq:Lyapunov-risk} (or the empirical risk function \eqref{eq:empirical-Lyapunov-risk}) consists of a single term only with a single parameter. This is in contrast to existing works \cite{chang2020neural,grune2020computing} where the risk functions have multiple terms to penalize the violations of the negative-orbital-derivative condition, positive definiteness condition, bound requirements, etc. Thus, network training in these works requires experienced users to carefully tune the hyper-parameters to properly weigh these penalty terms in order to obtain a reasonable solution. On the other hand, the proposed empirical risk function for Lyapunov-Net $V_{\theta}$ requires little effort in parameter-tuning, and the convergence is much faster in network training, as will be demonstrated in our numerical experiments below.

\subsection{Lyapunov-Net approximation and certification theory}
\label{subsec:certification}
We now provide theoretical guarantees on the approximation ability of Lyapunov-Net. We shall concern ourselves with the activation function RePU which in this section shall refer to the function RePU$(x)=\max\{0,x\}^2$. RePU is preferred to the common ReLU as it is $C^1$ and hence more useful to approximating Lyapunov functions. We note that the forthcoming theory can be extended to other even smoother activation functions (such as $\tanh$ or sigmoid). We present RePU in this paper due to its ability to exactly represent polynomials \cite{li2020repu} allowing for some proof simplification. 

Let $\Omega=[-1,1]^d$
and $\Omega_{\delta}=\Omega \setminus B(0;\delta)$. 
Let $x' = f(x)$ be the dynamical system defined by $f$. Further we denote $W^{k,p}(\Omega)$ as the regular Sobolev space over $\Omega$, and $C^k(\Omega)$ as the space of $k$-times continuously differentiable functions.

\noindent
\textbf{Assumption 1.}
$f$ is $L_f$-Lipschitz continuous on $\Omega$ for some $L_f>0$ and $f(0)=0$. 

\noindent
\textbf{Assumption 2.}
There exist $V^* \in C^{1}(\Omega;\Rbb)$ and constants $\alpha, \beta, \gamma >0$, such that for all $x \in \Omega$ there are
\begin{subequations}
    \begin{align}
    \alpha \|x\|  \le V^*(x) \le \beta \|x\|, \label{eq:V-squeeze}\\
    DV^*(x) \cdot f(x) \le - \gamma \|x\|. \label{eq:od-bound}
\end{align}
\end{subequations}
Assumption 1 implies that $\|f\|_{C(\Omega)} \le L_f \sqrt{d} < \infty$ and $f \in W^{1,\infty}(\Omega)$. 
Assumption 2 can also be relaxed to a more general case where $\alpha, \beta \in \Kcal_{\infty}$ and $\sigma \in \Lcal$ such that
\begin{align*}
    \alpha(\|x\|)  \le V^*(x) \le \beta(\|x\|) \\
    DV^*(x) \cdot f(x) \le - \rho (x)
\end{align*}
and $\rho$ is a positive definite function such that $\rho(x) \ge \alpha(\|x\|)\sigma(\|x\|)$. 
Here $\Kcal_{\infty}:=\{\alpha: \Rbb_{+} \to \Rbb_{+}:\alpha(0)=0, \lim_{s \to \infty}\alpha(s)=\infty,\, \mbox{$\alpha$ is strictly increasing and continuous}\}$ and $\Lcal:=\{\sigma :\Rbb_{+} \to \Rbb_{++} : \lim_{s\to \infty} \sigma(s)=0,\,\mbox{$\sigma$ is strictly decreasing and continuous}\}$ \cite{kellet2014compendium}. For simplicity, we consider \eqref{eq:V-squeeze} and \eqref{eq:od-bound} in this paper, which hold in a large class of real-world problems.
Assumption 2 also indicates that $V^*$ is a Lyapunov function and $0$ is the global asymptotically stable equilibrium in $\Omega$.

We will need the following result on universal approximation power of neural networks.
\begin{lemma}[Theorem 4.9 \cite{guhring2020approximation}]
Let $d \in \mathbb{N}$, $B>0$, and $\|g\|_{W^{1,\infty}(\Omega)}\leq B$. For all $\epsilon \in (0,1/2)$, there exists a feed-forward neural network $\phi_{\theta}$ with network parameter $\theta$ and RePU activation such that
\[
\|\phi_{\theta}-g\|_{W^{1,\infty}(\Omega)} \leq \epsilon,
\]
where the $W^{1,\infty}$-norm of a function $g$ is defined by
\[
\|g\|_{W^{1,\infty}(\Omega)} = \max\Big\{\esssup_{x\in\Omega} |g(x)|,\max_{1\le i\le d}\esssup_{x\in\Omega} |D_i g(x)| \Big\}
\]
and $D_i g(x)$ is the weak partial derivative of $g$ with respect to $x_i$ at $x$.
\label{lem:guhring}
\end{lemma}

We now show the above for the general Lyapunov-Net.
\begin{lemma} For any $\delta>0$,
the set of Lyapunov-Nets of form $V_{\theta}(x)=|\phi_{\theta}(x)-\phi_{\theta}(0)| + \bar{\alpha} \|x\|$ with RePU network $\phi_{\theta}$ and bounded weights $\theta \in [-1,1]^m$, where $m$ is the number of trainable parameters in $\phi_{\theta}$, is dense in the function space $\Scal:=\{\bar{g}(x)+\alpha \|x\|: \bar{g} \in C^1(\Omega;\mathbb{R}_{+}),\,\bar{g}(0)=0\}$ under the $W^{1,\infty}(\Omega_{\delta})$ norm so long as $\bar{\alpha} \in (0,\alpha)$.
\label{lem:LN-dense}
\end{lemma}
\begin{proof}
Let $h \in \Scal$ be arbitrary, where $h(x)=\bar{g}(x)+\alpha\|x\|$ for some $\bar{g}$ as characterized by the definition of $\Scal$. Further let $\delta>0$ and $\epsilon >0$. Let $\bar{\alpha}\in(0,\alpha)$ and define $g(x):=\bar{g}(x)+(\alpha-\bar{\alpha})\|x\|$. We will use $V_{\theta}(x)$ to approximate $h(x)=g(x)+\bar{\alpha}\|x\|$. We note that for all $x \in\Omega_{\delta}$ we have  $g(x)\geq a := \inf_{x \in \Omega_{\delta}} g(x) \ge (\alpha - \bar{\alpha}) \delta > 0$. From Lemma \ref{lem:guhring} we know there exists a RePU neural network $\phi_{\theta}: \mathbb{R}^d \to \mathbb{R}$ such that $\|\phi_{\theta}-g\|_{W^{1,\infty}(\Omega)} \leq \min(a/2,\epsilon/2$). Noting $g(0)=\bar{g}(0)=0$ we find that $
|\phi_{\theta}(0)| = |\phi_{\theta}(0) - g(0)|  \le a/2$. On the other hand, for all $x \in \Omega$, there is $|\phi_{\theta}(x) - g(x)| \le a/2$ which implies $\phi_{\theta}(x) \ge g(x) -a/2 \ge a/2$. Hence $\phi_{\theta}(x)-\phi_{\theta}(0) \geq 0$ for all $x \in \Omega_{\delta}$. Furthermore,
for all $x \in \Omega$, there is
\begin{align*}
|V_{\theta}(x)-h(x)|
&=|\phi_{\theta}(x)-\phi_{\theta}(0)-g(x)+g(0)|\\
&\leq|\phi_{\theta}(x)-g(x)|+|\phi_{\theta}(0)-g(0)|\\
&\leq \epsilon.
\end{align*}
From this we find
\begin{align*}
    |D_i V_{\theta}(x)- D_i h(x)| 
    & = |D_i(\phi(x)-\phi(0))-D_ig(x)|\\
    & \leq | D_i\phi(x)-D_i g(x) |\\
    & \leq \frac{\epsilon}{2},
\end{align*}
for all $x \in \Omega_{\delta}$ a.e. Thus
\[
\|V_{\theta}-h\|_{W^{1,\infty}(\Omega_{\delta})} \leq \max\big(\epsilon, \frac{\epsilon}{2}\big)=\epsilon.
\]

To show the above is true for RePU networks whose weight matrices are bounded we note the following observation: Suppose we have any layer of a RePU network of the form $RePU(A_{\ell}z_{\ell-1}+b_{\ell})$, where  $A_{\ell}\in \Rbb^{w_{\ell} \times w_{\ell-1}}$ and $b_{\ell} \in \Rbb^{w_{\ell}}$ with $w_{\ell}$ the width of layer $\ell$. If $A_{\ell}$ and $b_{\ell}$ have max entry $L$, then both can be decomposed into 
\begin{align*}
A_{\ell} = A^{(\ell)}_1+ \cdots + A^{(\ell)}_{N},\quad
b_{\ell} = b^{(\ell)}_1+\cdots+ b^{(\ell)}_N,
\end{align*}
such that each $A^{(\ell)}_i \in [-1,1]^{w_{\ell} \times w_{\ell-1}}$ and $b_i^{(\ell)} \in [-1,1]^{w_{\ell}}$. As RePU activation functions can exactly represent the identity function $Id$ using weights in $[-1,1]$ (see for example Lemma 2.1 \cite{li2020repu}), we rewrite the layer $RePU(A_{\ell}z_{\ell-1}+b_{\ell})$ as
\begin{align*}
 RePU(B_{\ell} Id(\bar{A}z_{\ell-1}+\bar{b}))&=RePU([A^{(\ell)}_1z_{\ell-1}+b^{(\ell)}_1]\\
 & \quad \quad +\cdots+[A^{(\ell)}_Nz_{\ell}+b^{(\ell)}_N])\\
 &=RePU(A_{\ell}z_{\ell-1}+b_{\ell}),   
\end{align*}
where $N \leq \lceil L \rceil$, $B_{\ell}=[I_{w_{\ell}}, I_{w_{\ell}},\ldots I_{w_{\ell}}]$ contains $N$ order ${w_{\ell}}$ identity matrices, $\bar{A}=[(A^{(\ell)}_1)^T,\ldots,(A^{(\ell)}_N)^T]^T$, and $\bar{b}=[(b^{(\ell)}_1)^T,\ldots,(b^{(\ell)}_N)^T]^T$. 

Combining the above observation and the proof for general RePU feed-forward networks completes the proof for some bounded weights $\theta \in [-1,1]^m.$
\end{proof}
In practice  the usefulness of bounding the weights of our network is that  once we have fixed the depth and size of Lyapunov-Net we can ensure the Lipschitz constant of $DV_{\theta}$ does not grow unbounded during training. This guarantees a robust verification method through use of this Lipschitz constant to ensure our approximating network is indeed a Lyapunov function after tuning for size and width. 
We note that while many others \cite{guhring2020sobolev,yarotsky2017error,petersen2018optimal,chen2019ANO,shen2020deep, shen2021neuralNA} consider the needed size of networks to approximate functions in certain spaces, there are much fewer results on Lipschitz bounded networks (See for example \cite{huster2019dnnLip, ugo2021groupsort} and references therein). As network size bounds usually do not reflect the practical reality of many applications of Lyapunov-Net, we will not exploit such analysis here.

From here on we assume $\Theta = [-1,1]^m$ for some $m \in \mathbb{N}$.

\begin{lemma}
\label{lem:exists-LN}
For any $\varepsilon \in (0, \frac{\gamma}{\sqrt{d}L_f})$ and $\delta >0$, there exists $\theta^* \in \Theta$ such that $V_{\theta^*}$ satisfies \eqref{eq:V-squeeze} and
\begin{equation}
\label{eq:LN-od-bound}
    DV_{\theta^*}(x) \cdot f(x) \le - a_{\gamma,\varepsilon} \|x\|, \quad \forall x \in \Omega_{\delta}.
\end{equation}
where $a_{\gamma,\varepsilon}:= \gamma - \varepsilon L_f >0$. 
%
\end{lemma}
\begin{proof}
By Lemma \ref{lem:LN-dense}, the set of Lyapunov-Nets, denoted by $\Vcal_{\Theta}^{\epsilon}$, forms an $\varepsilon$-net of $\Scal$ in the $W^{1,\infty}(\Omega_{\delta})$ sense. That is, for any $h \in \Scal$, there exists $\theta \in \Theta$ such that $V_{\theta} \in \Vcal_{\Theta}^{\varepsilon}$ and $\|V_{\theta} - h\|_{W^{1,\infty}(\Omega_{\delta})} < \varepsilon$. 
Since $V^* \in \Scal$, we know there exists $\theta^* \in \Theta$ such that $V_{\theta^*} \in \Vcal_{\Theta}^{\varepsilon}$ and $\|V_{\theta^*} - V^*\|_{W^{1,\infty}(\Omega_{\delta})} < \varepsilon$. 
Therefore
\begin{align*}
    DV_{\theta^*}(x) \cdot f(x) 
    & = DV^*(x) \cdot f(x) + (DV^*(x) - DV_{\theta^*}(x)) \cdot f(x) \\
    & \le - \gamma \|x\| + \|DV^*(x) - DV_{\theta}(x)\| \cdot \| f(x) \|\\
    & \le - \gamma \|x\| + \varepsilon \sqrt{d} L_f\|x\| \\
    & = -a_{\gamma,\varepsilon} \|x\|
\end{align*}
for all $x \in \Omega_{\delta}$, where we used the facts that $\|V_{\theta} - V^*\|_{W^{1,\infty}(\Omega_{\delta})} < \varepsilon$ and $f$ is $L_f$-Lipschitz in $\Omega$, and $f(0)=0$ to obtain the last inequality.
\end{proof}
With the set of bounded weights $\Theta$ and fixed network $V_{\theta}$, we know there exists $M>0$ such that $DV_{\theta}(\cdot)\cdot f(\cdot)$ is $M$-Lipschitz on $\Omega$ for all $\theta \in \Theta$.
Moreover, as we discussed above, minimizing an empirical loss function based on finitely many sample collocation points cannot guarantee to obtain a Lyapunov function. Hence we provide a theoretical guarantee on finding an $\delta$-accurate Lyapunov function.

\begin{lemma}
\label{lem:pts-cover}
For any $\delta, c \in (0,1)$, there exists $N=N( \delta, c) \in \Nbb$, such that for some $N$ sampling points $x^{(i)} \in \Omega_{\delta}$ where $i=1,\dots,N$, we have $\Omega_{\delta} \subset \cup_{i=1}^{N}B(x^{(i)}; c\|x^{(i)}\|)$.
\end{lemma}
\begin{proof}
Define the radi $r_1,\ldots,r_k,r_{k+1}$ such that $\delta/\sqrt{d}=r_1<r_2<\cdots<r_k \le 1<r_{k+1}$ and $r_{i+1}-r_i=cr_i$. 
The above relation requires $r_i=(1+c)^{i-1}\delta/\sqrt{d}$. Thus $k = \lfloor -\ln(\delta/\sqrt{d})/\ln(1+c)\rfloor + 1$. Now define $I=\{\pm r_j: j \in [k] \}$ where $[k]:=\{1,\dots,k\}$ and let $X = \{x=(x_1,\dots,x_d)\in \Omega: x_i\in I,\, i \in [d] \}$. Hence we set $N = |X|=(2k)^d$.

Let $y \in \Omega_{\delta}$ be arbitrary. Then for each component $y_j$ of $y$ we know there exists integer $k_j \in[k]$ such that $|y_j| \in [r_{k_j},r_{k_j+1}]$. Therefore, we can choose the grid point $x=(\text{sign}(y_1)r_{k_1},\ldots,\text{sign}(y_d) r_{k_d}) \in X$ for which 
\begin{align*}
\|x-y\|^2&=(r_{k_1}-|y_1|)^2+\cdots+(r_{k_d}-|y_d|)^2\\
&\leq (r_{k_1}-r_{k_1+1})^2+\cdots+(r_{k_d}-r_{k_d+1})^2\\
&= (cr_{k_1})^2+\cdots+(cr_{k_d})^2\\
&\leq c^2\|x\|^2.
\end{align*}
Thus $\|x-y\|\leq c\|x\|$. 
As $y \in \Omega_{\delta}$ is arbitrary, we know $\Omega_{\delta} \subset \cup_{i=1}^{N}B(x^{(i)}; c\|x^{(i)}\|)$.
\end{proof} 
\begin{theorem}
For any $ \delta\in (0,1)$ and $a_{\gamma,\epsilon}$ as defined in Lemma \ref{lem:exists-LN}, choose any $\bar{\gamma} \in (0,a_{\gamma,\epsilon})$ and $0<c< \frac{\bar{\gamma}}{M} < 1$ (if $\frac{\bar{\gamma}}{M}\ge 1$, then choose any $c \in (0,1)$). Let $N$ and $X:=\{x^{(i)}: i\in[N]\}$ be given as in Lemma \ref{lem:pts-cover} and the empirical risk function $\hat{\ell}(\theta) = \frac{1}{N} \sum_{i=1}^{N}(DV_{\theta}(x^{(i)})\cdot f(x^{(i)}) + \bar{\gamma}\|x^{(i)}\|)_+$. Then minimizer of $\hat{\ell}$ must exist and achieve minimum function value $0$. Moreover, for any minimizer $\hat{\theta}$ of $\hat{\ell}(\cdot)$, the corresponding $V_{\hat{\theta}}$ is an $\delta$-accurate Lyapunov function of $f$ on $\Omega$.
\end{theorem}
\begin{proof}
By Lemma \ref{lem:exists-LN}, we know there exists $\theta^* \in \Theta$ such that $V_{\theta^*}$ is positive definite and $DV_{\theta^*}(x)\cdot f(x) \le - a_{\gamma,\epsilon}\|x\| \le - \bar{\gamma}\|x\|$ for all $x \in \Omega_{\delta}$. Hence $\hat{\ell}(\theta^*)=0$. Given that $\hat{\ell}(\theta)$ is nonnegative for any $\theta$, we know the minimum value of $\hat{\ell}$ on $\Theta$ is $0$. Let $\hat{\theta}$ be any minimizer of $\hat{\ell}$, then there is also $\hat{\ell}(\hat{\theta}) = 0$, which implies $DV_{\hat{\theta}}(x) \cdot f(x) \le - \bar{\gamma}\|x\|$ for all $x \in X$.

Now for any $y \in \Omega_{\delta}$, there exists $x \in X$, such that $y \in B(x;c\|x\|)$ due to Lemma \ref{lem:pts-cover}. Hence
\begin{align*}
    DV_{\hat{\theta}}(y) \cdot f(y) 
    & \le DV_{\hat{\theta}}(x) \cdot f(x) + |DV_{\hat{\theta}}(y)\cdot f(y) - DV_{\hat{\theta}}(x)\cdot f(x) | \\
    & \le -\bar{\gamma}\|x\| + M \|y - x\|\\
    & \le -\bar{\gamma}\|x\| + M c \|x\|\\
    & \le -(\bar{\gamma}- M c) \delta.
\end{align*}
Since $y$ is arbitrary, we know $V_{\hat{\theta}}$ is a $\delta$-accurate Lyapunov function.
\end{proof}

\subsection{Application to control and others}
\label{subsec:application}
In light of the power of Lyapunov functions, we can employ the proposed Lyapunov-Net to many control problems of nonlinear dynamical systems in high-dimension. In this subsection, we instantiate one of such applications of Lyapunov-Net to approximate control Lyapunov function. 

Consider a nonlinear control problem $x' = f(x,u)$ where $u: \mathbb{R}^{d} \to \mathbb{R}^{n}$ ($n$ is the dimension of the control variable at each $x$) is an unknown state-dependent control in order to steer the state $x$ from any initial to the equilibrium state $0$.
To this end, we parameterize the control as a deep neural network $u_{\eta}: \mathbb{R}^{d} \to \mathbb{R}^{n}$ (a neural network with input dimension $d$ and output dimension $n$) where $\eta$ represents the network parameters of $u_{\eta}$. 
In practice, the control variable is often restricted to a compact set in $\mathbb{R}^d$ due to physical constraints.
This can be easily implemented in a neural network setting.
For example, if the magnitude of the control is required to be in $[-\beta,\beta]$ componentwisely, then we can simply apply $\beta \cdot \tanh(\cdot)$ to the last, output layer of $u_{\eta}$.

Once the network structure $u_{\eta}$ is determined, we can define the risk of the control-Lyapunov function (CLF):
\begin{equation}
    \label{eq:control-Lyapunov-risk}
    \ell_{\text{CLF}} (\theta,\eta) := \frac{1}{|\Omega|} \int_{\Omega} ( DV_{\theta}(x)\cdot f(x, u_{\eta}(x))  +\gamma\|x\|)_+^2 \, dx.
\end{equation}
Minimizing \eqref{eq:control-Lyapunov-risk} yields the optimal parameters $\theta$ and $\eta$. In practice, we again approximate $\ell_{\text{CLF}}(\theta,\eta)$ by its empirical expectation $\hat{\ell}_{\text{CLF}}(\theta,\eta)$ at sampled points in $\Omega$, as an analogue to $\ell(\theta)$ versus $\hat{\ell}(\theta)$ above.
Then the minimization can be implemented by alternately updating $\theta$ and $\eta$ using (stochastic) gradient descent on the empirical risk function $\hat{\ell}_{\text{CLF}}$. Similar as \eqref{eq:Lyapunov-risk}, we have a single term in the loss function in \eqref{eq:control-Lyapunov-risk}, which does not have hyper-parameters to tune and the training can be done efficiently.
\section{NUMERICAL EXPERIMENTS}
\label{sec:results}
\subsection{Experiment setting}
We demonstrate the effectiveness of the proposed method through a number of numerical experiments in this section.
In our experiments, the value of $\bar{\alpha}$ used in \eqref{eq:V} and the depth and size of $\phi_{\theta}$ used in $V_{\theta}$ for the three test problems are summarized in Table \ref{tab:parameters}.
We minimize the empirical risk function $\hat{\ell}$ using the Adam Optimizer with learning rate $0.005$ and $\beta_1=0.9$, $\beta_2=0.999$ and Xavier initializer.
In all tests, we iterate until the associated risk of \eqref{eq:empirical-Lyapunov-risk} is below a prescribed tolerance of $10^{-4}$. We use a sample size $N$ (values shown in Table \ref{tab:parameters}), i.e., the number of sampled points in $\Omega$ in \eqref{eq:empirical-Lyapunov-risk}, such that the associated risk reduces reasonably fast while maintaining good uniform results over the domain. We note these points are drawn using a uniform sampling method for simplicity. We finally note in all experiments the weights remain bounded in $[-1,1]^m$ as noted in our theory.

%
All the implementations and experiments are performed using PyTorch in Python 3.9 in Windows 10 OS on a desktop computer with an AMD Ryzen 7 3800X 8-Core Processor at 3.90 GHz, 16 GB of system memory, and an Nvidia GeForce RTX 2080 Super GPU with 8 GB of graphics memory.
The number of iterations needed to reach our stopping criteria in \eqref{eq:empirical-Lyapunov-risk} and training time (in seconds) for the three tests are also given in Table \ref{tab:parameters}. As discussed in Section \ref{sec:related}, how the width and depth of a neural network is related to approximation power is an area of active research. In our experiments, we manually selected the width and depth as shown in Table \ref{tab:parameters} which yield good results. We leave investigation of how these affect the approximation power of Lyapunov-Net to further study.

\begin{table}[ht]
\caption{Network parameter setting and training time in the tests.}
\label{tab:parameters}
\vspace{-12pt}
\begin{center}
\begin{tabular}{crrccc}
\toprule
Test Problem & Time & Iter. & $N$ & Depth/Width & $\bar{\alpha}$ \\
\midrule
Curve Tracking & 0.5 s&  2 & 100K & 3/10 & 0.5 \\
10d Synthetic DS & 1.5 s& 2 & 200K & 3/20 & 0.01 \\
30d Synthetic DS & 2.5 s& 15 & 400K & 5/20 & 0.01 \\
\bottomrule
\end{tabular}
\end{center}
\vspace{-12pt}
\end{table}

\vspace{-7pt}

\subsection{Experimental results}
To demonstrate the effectiveness of the proposed method, we apply Lyapunov-Net \eqref{eq:V} to three test problems: a two-dimensional (2d) nonlinear system from the curve-tracking application \cite{mukhopadhyay2020algorithm}, a 10d and 30d synthetic dynamical system (DS) from \cite{grune2020computing}.


\paragraph{2d DS in curve tracking} 
We apply our method to find the Lyapunov function for a 2d nonlinear DS in a curve-tracking problem \cite{mukhopadhyay2020algorithm}. The DS of $x=(\rho,\varphi) $ is given by
\begin{subequations}
    \begin{align}
    \Dot{\rho} &= - \sin (\varphi), \label{eq:curve-tracking-rho}\\
    \Dot{\varphi} &= (\rho - \rho_0) \cos (\varphi) - \mu \sin( \varphi) + e.  \label{eq:curve-tracking-phi}
\end{align} 
\end{subequations}
We use the following constants in our experiments: $e = 0.15$, $\rho_0 =1$, and $\mu = 6.42$ from \cite{mukhopadhyay2020algorithm} as well as RePU activation. 

%
%
The problem domain was mapped to Cartesian coordinates around critical point $x^*=(1,0)$, and converted back when the training is done. The result shows that after only 2 iterations the positive definiteness and negative-orbital-derivative conditions are met.
We used RePU activation which is consistent with the theoretical results provided in Section III-C. However, these results still hold by readily modifying the proofs if the common tanh activation is used. We have also tested tanh activation in Lyapunov-Net and obtained similar performance.

\paragraph{10d Synthetic DS} We consider a 10d synthetic DS from \cite{grune2020computing}, which is a vector field defined on $[-1,1]^{10}$ and has an equilibrium at $0$. 
%
%
%
The iteration number and computation time needed for training are shown in Table \ref{tab:parameters}. 

\paragraph{30d Synthetic DS} We consider the same system as the 10d Synthetic DS but concatenate the function $f$ three times to get a 30d problem.
%
%
Figure \ref{fig:10d_x2_x8_relu} graphed the approximated Lyapunov function $V_{\theta}$ (top solid) and $D V_{\theta} \cdot f$ (bottom wire) in the $(x_2,x_8)$ and $(x_{10}, x_{13})$ planes, using RePU activation.
These plots show that Lyapunov-Net can effectively approximate Lyapunov functions in such high-dimensional problem. We shall also use this example to compare the convergence speeds to the Neural Network structures of \cite{chang2020neural, grune2020computing} in the following section.
%
%
\begin{figure}[thpb]
  \centering
    \includegraphics[width=0.2\textwidth]{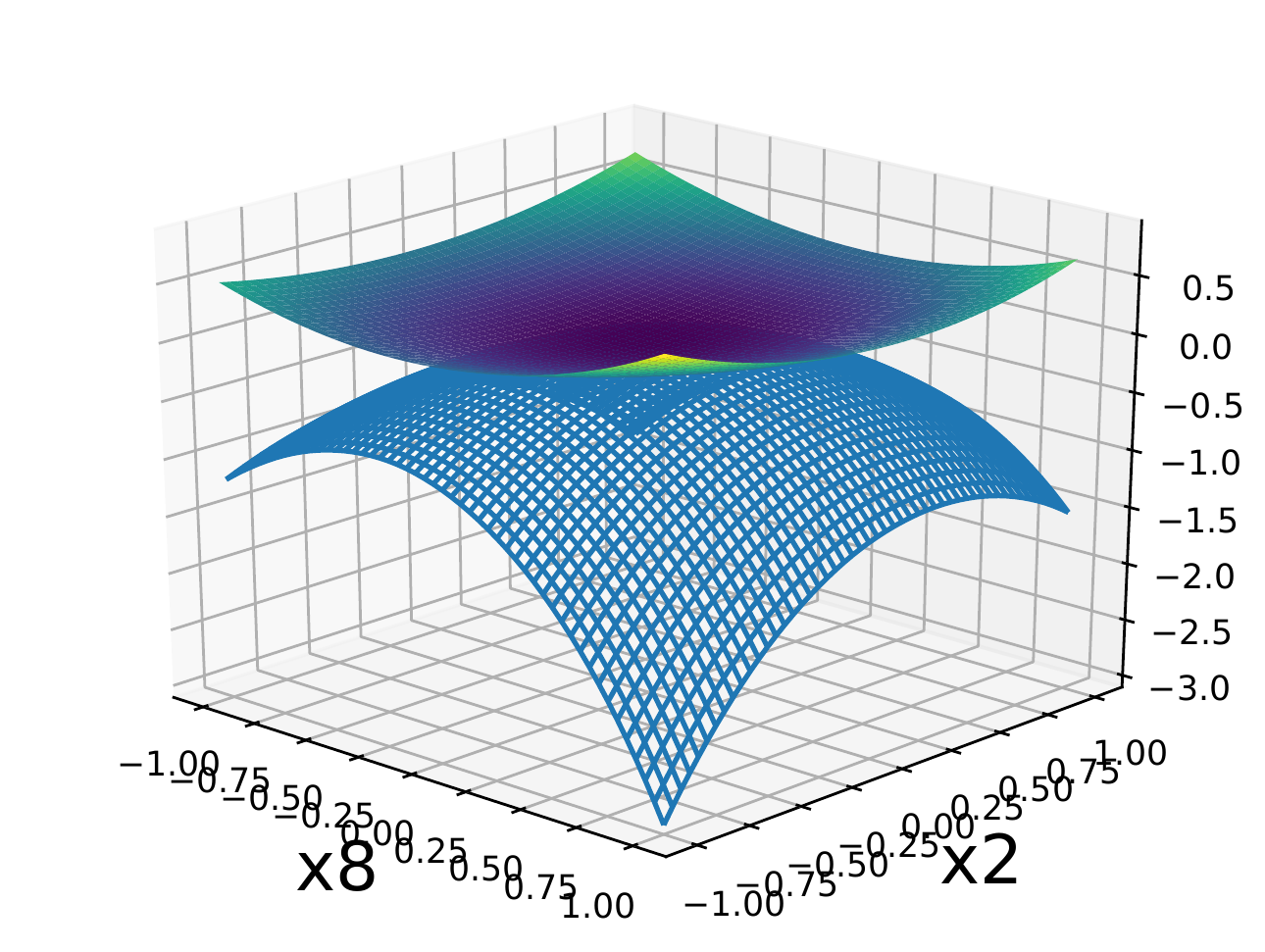}
    \includegraphics[width=0.2\textwidth]{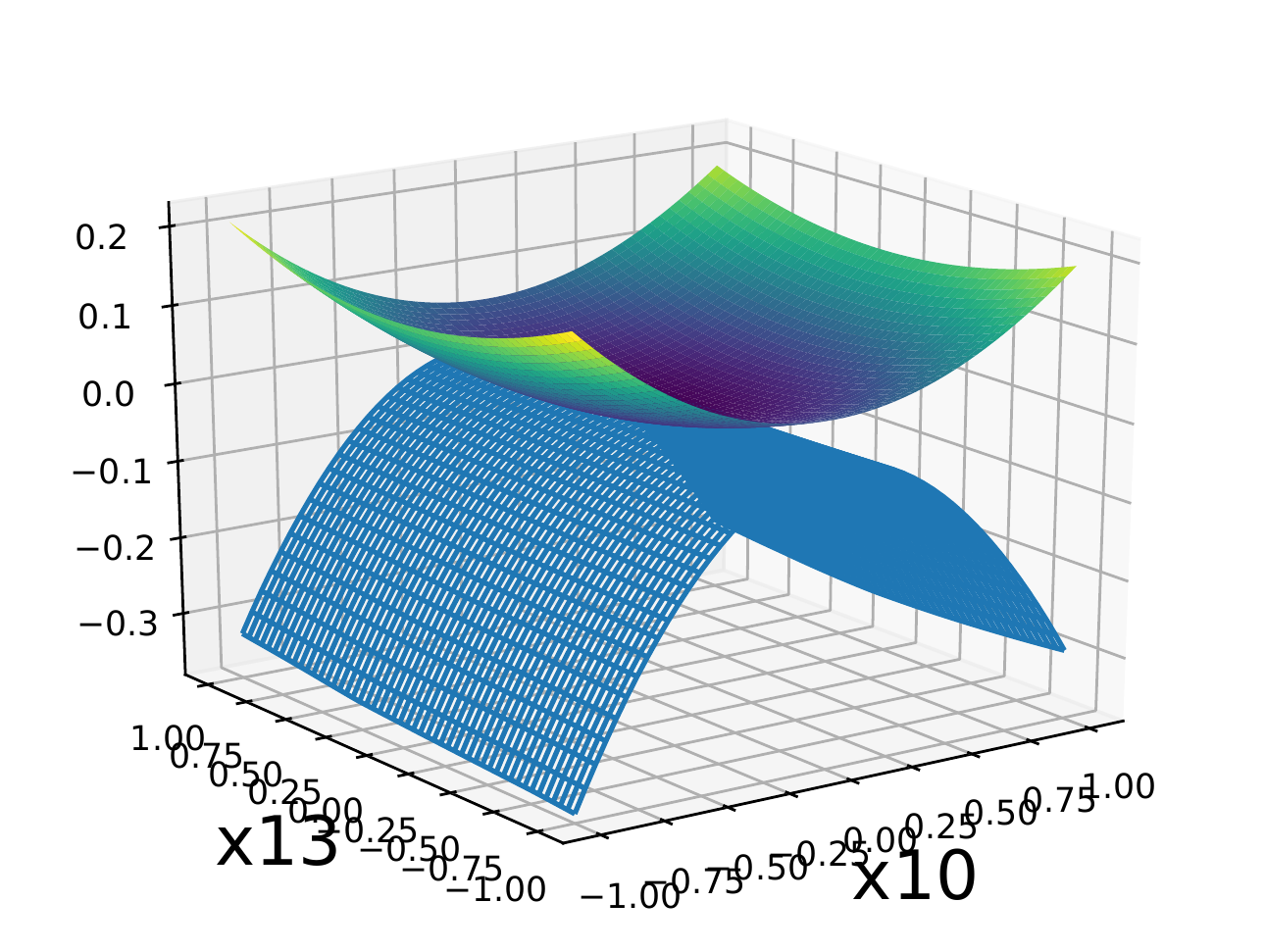}
  \caption{Plot of the learned Lyapunov-Net $V_{\theta}$ (top solid) and $D V_{\theta} \cdot f$ (bottom wire) in the $(x_2,x_8)$-plane (left) and $(x_{10},x_{13})$-plane (right) for the 30d synthetic DS.}
  \label{fig:10d_x2_x8_relu}
  \vspace{-12pt}
\end{figure}

\vspace{-3pt}

\subsection{Comparison with existing DL methods}
\label{sec:comparison}

To demonstrate the significant improvement of Lyapunov-Net over existing DL methods in approximation efficiency, we compare the proposed Lyapunov-Net to two recent approaches \cite{chang2020neural,grune2020computing} that also use deep neural networks to approximate Lyapunov functions in continuous DS setting. Specifically, we use the 30d synthetic DS described above as the test problem in this comparison. 
Both \cite{chang2020neural,grune2020computing} employ generic deep network structure of $V_{\theta}$, and thus require additional terms in their risk functions to enforce positive definiteness of the networks. 
%
%
Specifically, the following risk function is used from \cite{grune2020computing}:
\begin{align}
\label{eq:grune_loss}
    \hat{\ell}_1(\theta)=& \textstyle{\frac{1}{N}\sum\limits_{i=1}^{N}}\big(\left(DV^{DL}_{\theta}(x_i) \cdot f(x_i)+\|x_i\|^{2}\right)_{+}^{2}\\
    &+\left(20\|x_i\|^2 - V^{DL}_{\theta}(x_i)\right)_{+}^{2} +\left(V^{DL}_{\theta}(x_i)-0.2\|x_i\|^2\right)_{+}^{2}\big), \nonumber
\end{align}
which aims at an approximate Lyapunov function $V^{DL}_{\theta}$ satisfying $0.2\|x\|^2 \le V^{DL}_{\theta}(x) \le 20\|x\|^2$ and $DV^{DL}_{\theta}(x) \cdot f(x) \le -\|x\|^2$ for all $x$.
In \cite{chang2020neural}, the following risk function is used:
\begin{equation}
    \textstyle{\hat{\ell}_2(\theta)=V^{NL}_{\theta}(0)^2+\frac{1}{N}\sum\limits_{i=1}^{N}[(D V^{NL}_{\theta}(x_i) \cdot f(x_i))_+
    +(-V^{NL}_{\theta}(x_i))_+]},
\label{eq:chang_loss}
\end{equation}
which aims at an approximate Lyapunov function $V^{NL}_{\theta}$ such that $V^{NL}_{\theta}(0)=0$, $V^{NL}_{\theta}(x) \ge 0$ and $DV^{NL}_{\theta}(x) \cdot f(x) \le 0$ for all $x$. 
The activation functions are set to softmax in \eqref{eq:grune_loss} as suggested in \cite{grune2020computing} and tanh in \eqref{eq:chang_loss} as suggested in \cite{chang2020neural}. We shall label these models as Deep Lyapunov (DL) and Neural Lyapunov (NL) respectively.
%
%
For Lyapunov-Net we use \eqref{eq:empirical-Lyapunov-risk} as it already satisfies the positive definiteness condition. 
We again do not impose any structural information of the problem into our training, and thus all test methods recognize the DS as a generic 30d system for sake of a fair comparison.

%
We use the value of the less stringent empirical risk function $\hat{\ell}_2$ with $N=400,000$ defined in \eqref{eq:chang_loss} as a metric to evaluate all three methods.
Specifically, we plot the values of $\hat{\ell}_2$ (in log scale) versus iteration number and wall-clock training time (in seconds) in Figure \ref{fig:30d_comparison} using the same learning rate for all methods.

\begin{figure}[thpb]
  \centering
  \includegraphics[width=0.2\textwidth]{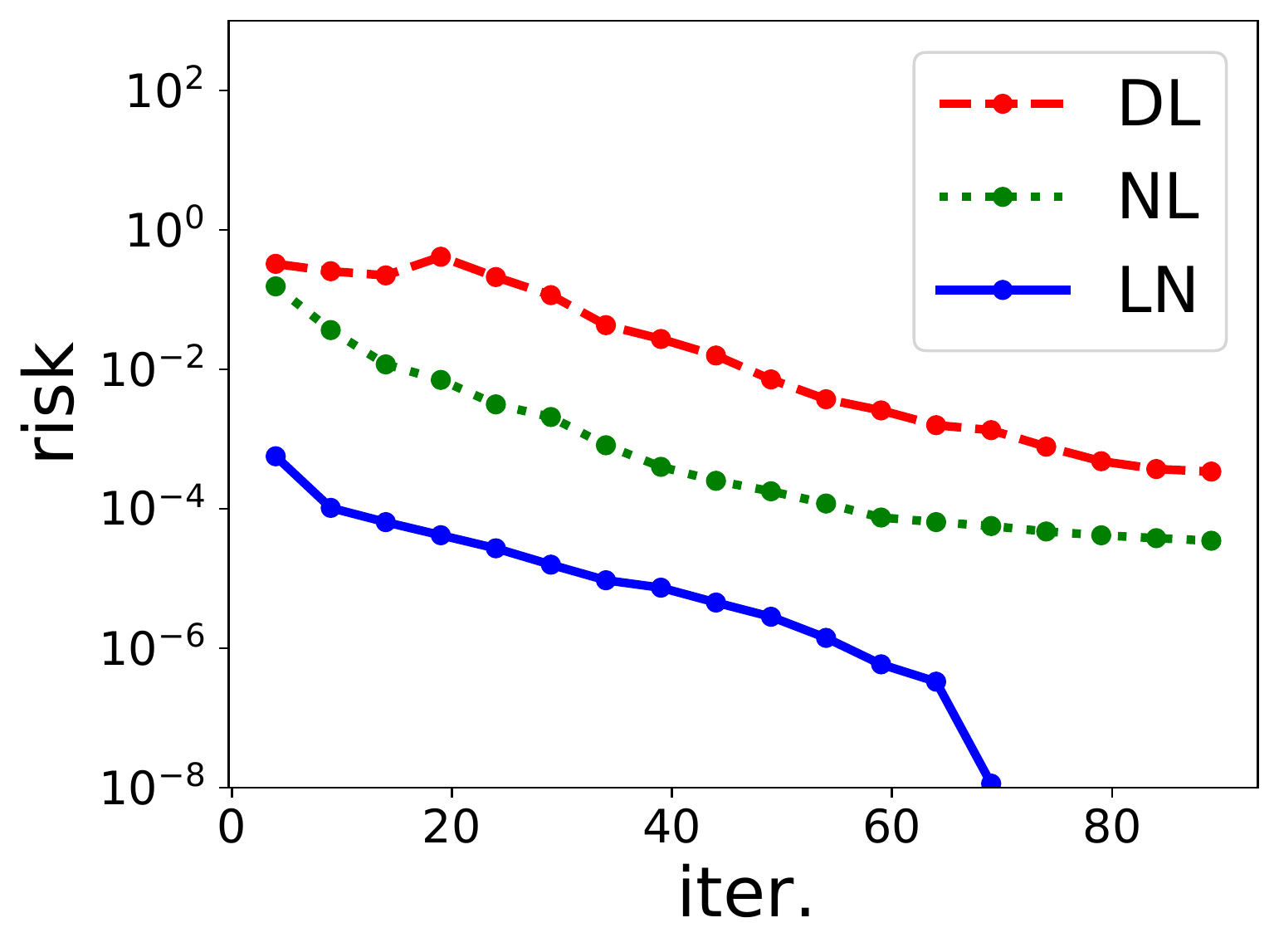}
   \includegraphics[width=0.2\textwidth]{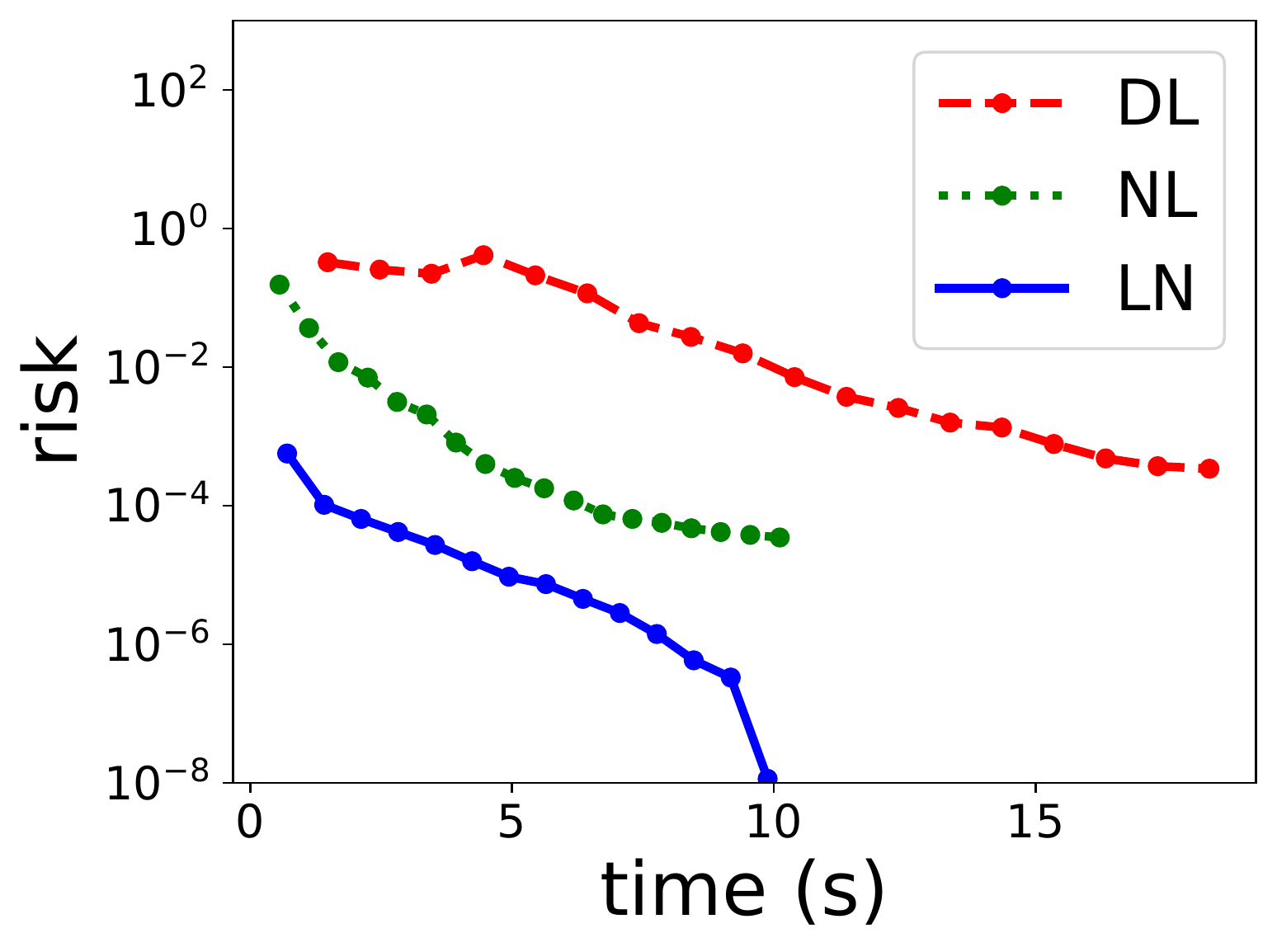}
  \caption{The value of empirical risk function $\hat{\ell}_2$ (in log scale) defined in \eqref{eq:chang_loss} versus iteration number (left) and wall-clock training time in seconds (right) for the three network and training settings: DL \cite{grune2020computing} (red dashed), NL \cite{chang2020neural} (green dotted), and Lyapunov-Net denoted LN (blue solid). Note that the risk of Lyapunov-Net often hits 0, the most desired risk function value, explaining the dropoff.} 
  \label{fig:30d_comparison}
  \vspace{-12pt}
\end{figure}

%
%
%
We see in Figure \ref{fig:30d_comparison} that Lyapunov-Net (LN) has risk value decaying much faster than the other methods. Further, Lyapunov-Net training does not need hyper-parameter tuning to achieve this speed, whereas DL and NL require careful and tedious tuning to balance the different terms in the risk function in order to achieve satisfactory results as shown in Figure \ref{fig:30d_comparison}. This highlights the efficiency and simplicity of Lyapunov-Net to find the desired Lyapunov functions. 

We note that the performance of all methods can be further improved using additional structural information as discussed in  \cite{grune2020computing} and falsification techniques in training as in \cite{chang2020neural}. We leave these improvements to future investigations.

\subsection{Application in Control}
\label{sec:control_comparison}
In this test we compare Lyapunov-Net and SOS/SDP methods in the problem of estimating the Region of Attraction (ROA) for the classical inverted pendulum control problem. We shall show how the Lyapunov-Net framework allows simultaneous training of the control policy and outperforms SOS/SDP methods by producing a larger ROA. 

The inverted pendulum is an often considered problem in control theory see \cite{chang2020neural, richards2018lyapunov}. For this problem we have dynamics $x=(\theta, \Dot{\theta})$ governed by $
    m l^{2} \ddot{\theta}=m g l\sin \theta-\beta \dot{\theta}+u$,
where $u=u(x)$ is the control. We use $g=9.82$, $l = 0.5$, $m=0.15$, and $\beta = 0.1$ for our experiments. We shall use this model to compare SOS/SDP type methods to Lyapunov-Net when it comes to estimating the ROA of this problem using the Lyapunov candidate function they both find within respective valid regions.
 
We adjust our model slightly so that we learn a $u$ policy alongside $V_{\theta}$. This is done by implementing a single layer neural network $\bar{u}$ which performs a simple linear transform for some learned matrix $A$. We use such a simple control so that the control can be used in the SOS/SDP algorithm. In the absence of such a comparison we could make $\bar{u}$ a much more complex and further improved control policy. 
%
%
The training method is the same for both networks. 
 
 \vspace{-10pt}
 
\begin{figure}[thpb]
  \centering
  \includegraphics[width=0.25\textwidth]{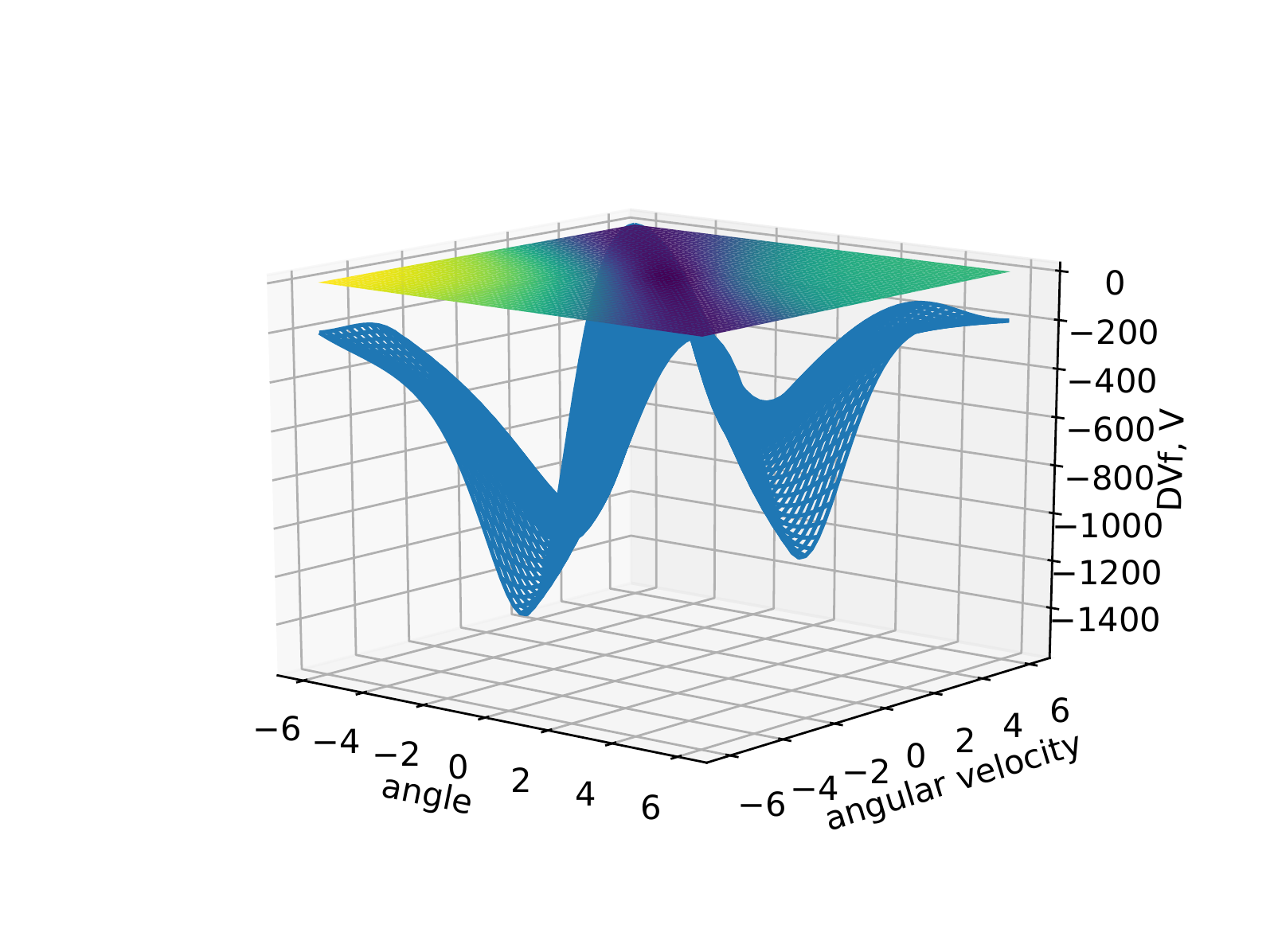}
  \includegraphics[width=0.17\textwidth]{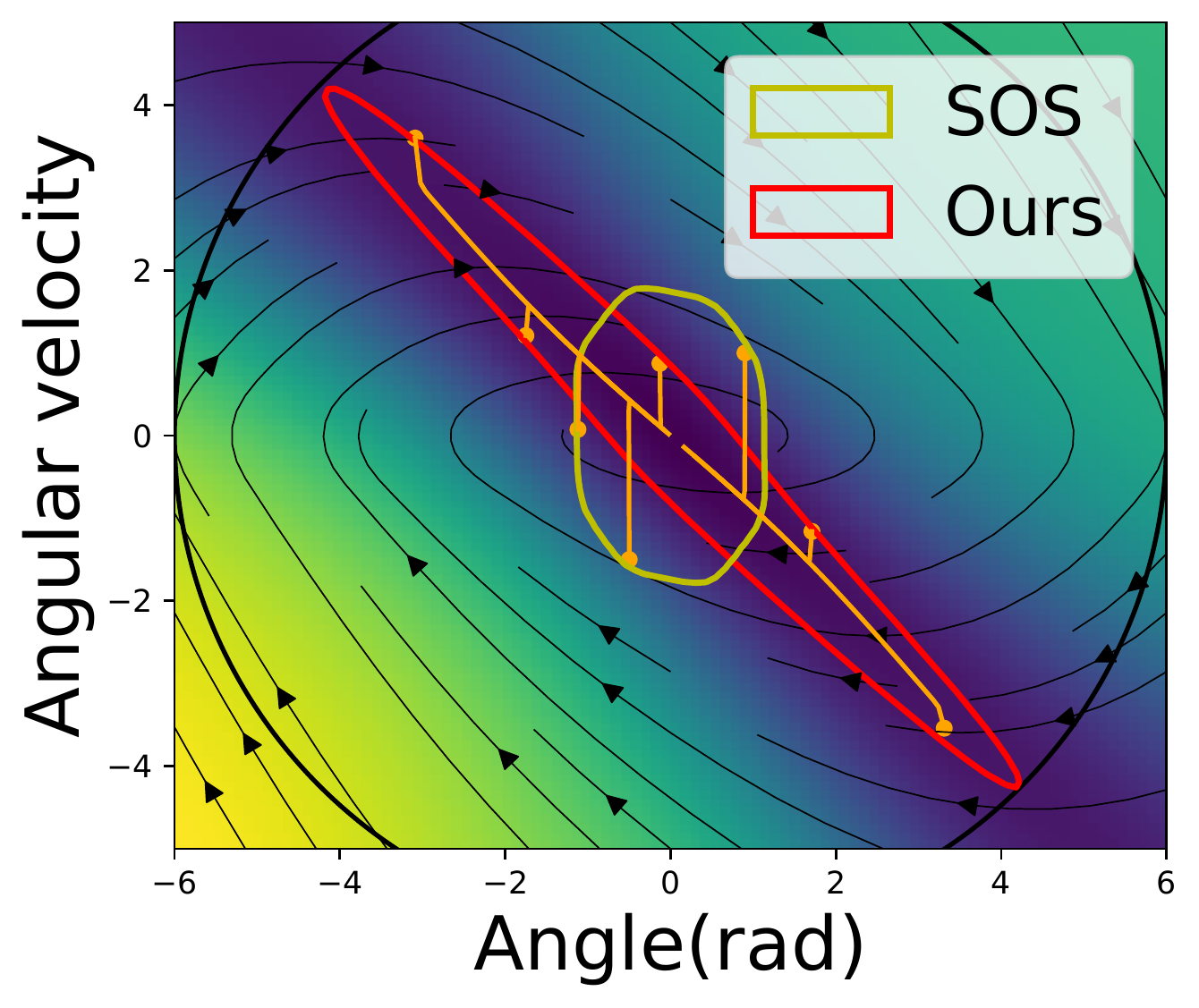}
  \caption{Left: learned Lyapunov-Net $V_{\theta}$ (top solid) and $D V_{\theta} \cdot f$ (bottom wire) for the 2d inverted pendulum in phase and angular velocity space. Right: Comparison of ROA for standard SOS/SDP methods and for Lyapunov-Net with orange sample trajectories.}
  \label{fig:inverted_3d_plot}
\end{figure}

\vspace{-9pt}

In Figure \ref{fig:inverted_3d_plot}, we use $\tanh$ as activation in $V_{\theta}$ because while the error bounds are similar with RePU, $\tanh$ has slightly better performance in this case. Once trained the $u$-policy found is used for the SOS type method. As SOS methods are only applicable to polynomial systems, we apply a Taylor approximation to the dynamics of the system and then compute a sixth order polynomial candidate using the standard SOS/SDP approach. We finally compute the ROA determined by both algorithms using the level-sets of the functions over the valid regions they found. We find that the area of the region found by Lyapunov-Net is larger than that found by SOS. This result is plotted in Figure \ref{fig:inverted_3d_plot} where the orange paths are a few sample trajectories. We note the similarity of our results to \cite{chang2020neural}, who in addition to the above, found examples where SOS methods produced incorrect ROA regions on this same problem.
%


\section{CONCLUSIONS}
\label{sec:conclusion}

We constructed a versatile deep neural network architecture called Lyapunov-Net to approximate Lyapunov function for general high-dimensional dynamical systems. We provided theoretical justifications on approximation power and certificate guarantees of Lyapunov-Nets. Applications to control Lyapunov functions are also considered. We demonstrated the effectiveness of our method on several test problems. The Lyapunov-Net framework developed in the present work is expected to be applicable to a much broader range of control and stability problems.

\addtolength{\textheight}{-13cm}   

\bibliographystyle{abbrv}
\bibliography{library}

\end{document}